\documentclass[12pt]{amsart} 
\usepackage[utf8]{inputenc} 
\usepackage{amsthm}
\usepackage{amssymb}
\usepackage{fancyref}
\usepackage{graphicx} 
\usepackage{array} 

\usepackage{ dsfont }
\usepackage{color, soul}
\usepackage{ bbold }
\usepackage{multirow}

\usepackage[utf8]{inputenc}
\usepackage[english]{babel}
\usepackage{url}

\newtheorem{theorem}{Theorem}[section]

\newtheorem{prop}{Proposition}[section]
\newtheorem{rem}{Remark}[section]
\newtheorem{cor}{Corollary}[section]
\newtheorem{ex}{Example}[section]

\usepackage[left=3cm,right=3cm,top=3cm,bottom=3cm]{geometry}
\usepackage{fancyhdr} 
\newcommand{\ee}{\varepsilon}

\newcommand{\EE}{\mathbb{E}}

\newcommand{\PP}{\mathbb{P}}
\newcommand{\RR}{\mathbb{R}}
\newcommand{\dint}{\mathrm{d}}

\begin{document}

\title{Stochastic geometry to generalize the Mondrian Process}
\author{Eliza O'Reilly and Ngoc Tran}



\email{eoreilly@caltech.edu,ntran@math.utexas.edu}
\thanks{We would like to thank James Murphy for providing the code used to simulate the STIT Tessellations.}

\maketitle

\begin{abstract}
 The stable under iterated tessellation (STIT) process is a stochastic process that produces a recursive partition of space with cut directions drawn independently from a distribution over the sphere. The case of random axis-aligned cuts is known as the Mondrian process. Random forests and Laplace kernel approximations built from the Mondrian process have led to efficient online learning methods and Bayesian optimization. In this work, we utilize tools from stochastic geometry to resolve some fundamental questions concerning STIT processes in machine learning. First, we show that a STIT process with cut directions drawn from a discrete distribution can be efficiently simulated by lifting to a higher dimensional axis-aligned Mondrian process. Second, we characterize all possible kernels that stationary STIT processes and their mixtures can approximate. We also give a uniform convergence rate for the approximation error of the STIT kernels to the targeted kernels, generalizing the work of [3] for the Mondrian case. Third, we obtain consistency results for STIT forests in density estimation and regression. Finally, we give a formula for the density estimator arising from an infinite STIT random forest. This allows for precise comparisons between the Mondrian forest, the Mondrian kernel and the Laplace kernel in density estimation. Our paper calls for further developments at the novel intersection of stochastic geometry and machine learning.
\end{abstract}

\section{Introduction}
The stable under iteration tessellation (STIT) \cite{Nagel2003,Nagel2005,Nagel2008} is a stochastic process that recursively generates self-similar stationary random partition of $\RR^d$. Each finite measure $\Lambda$ on the unit sphere in $\RR^d$ gives rise to a unique STIT distribution. 
The Mondrian process, put forward independently by Roy and Teh \cite{roy2008mondrian}, is a STIT with axis-aligned cut directions, that is, the measure $\Lambda$ is concentrated on the $d$ unit vectors $e_1,\dots,e_d$. Subsequent generalizations of the Mondrian process to oblique cuts  \cite{pmlr-v84-fan18b,TehRTFs2019} are also special cases of STIT.

From the viewpoint of machine learning, the Mondrian process was put forward as a spatial generalization of the stick-breaking process, one that underpins much of nonparametric Bayesian methods for clustering \cite{pitman2006combinatorial, teh2005sharing,dunson2008kernel,ishwaran2001gibbs,ren2011logistic,broderick2012beta,thibaux2007hierarchical,teh2007stick}.  
Lakshiminaraynan, Roy and Teh gave two use cases for the Mondrian process: kernel approximation \cite{balog2016mondrian} and random forests \cite{lakshminarayanan2014mondrian}. In \cite{balog2016mondrian}, they gave a uniform convergence rate of the Mondrian kernel to the Laplace, and demonstrated its computational advantages as a random kernel approximation method. In \cite{lakshminarayanan2014mondrian}, they showed that random forests constructed from the Mondrian process are easy to simulate, parallelizable, and can be updated online. Mondrian forests thus yield considerable computational advantages while giving similar performance to other online random forest methods \cite{lakshminarayanan2014mondrian}. In addition, the probabilistic construction of the Mondrian process yields a simple posterior update under hierarchical Gaussian prior \cite{lakshminarayanan2016mondrian}. Wang, Gehring, Kohli, and Jegelka \cite{wang2018batched} exploited this feature of the Mondrian to develop efficient Bayesian optimization algorithms. Mourtada, Ga\"{i}ffas, and Scornet  \cite{mourtada2017universal, mourtada2020minimax} showed that Mondrian forests yield randomized classification and regression algorithms that achieve minimax rates for the estimation of a $s$-H\"{o}lder regression function for $s \in (0,2]$. 

What does one gain by working with the more general STIT processes instead of the Mondrian? There could be computational gains, as shown by Ge, Wang, Teh, Wang and Elliott
\cite{TehRTFs2019} through classification task and a simulation study. In this paper, we demonstrate that there are definitive theoretical gains. 
By using tools from stochastic geometry, we obtain a series of theorems concerning kernel and density estimators using STIT processes. We hope that these results will spark further developments between stochastic geometry and machine learning. 

We now give an informal overview of our results. Theorem \ref{thm:projection} shows that a STIT process with cuts drawn from any finite collection of directions are simply projections of the classical axis-aligned Mondrian. This significantly simplifies the analysis of Mondrian processes with general directions and hints that the algorithm of \cite{TehRTFs2019} can be improved. Theorems \ref{thm:kernelcharacterization} and \ref{thm:monotone} completely characterize all possible kernels one can approximate by STIT processes with fixed and random lifetime, respectively. First introduced by Rahimi and Recht in \cite{Rahimi}, random features have been shown to greatly increase computational efficiency for learning algorithms with kernels.
Theorem \ref{thm:rate} gives a uniform convergence rate of the STIT approximations to the targeted kernels. Collectively, these theorems generalize the work of \cite{balog2016mondrian} from the Mondrian to arbitrary STIT.

\begin{figure}[h!]
    \centering
\begin{minipage}{.24\linewidth}
  \includegraphics[width=\linewidth]{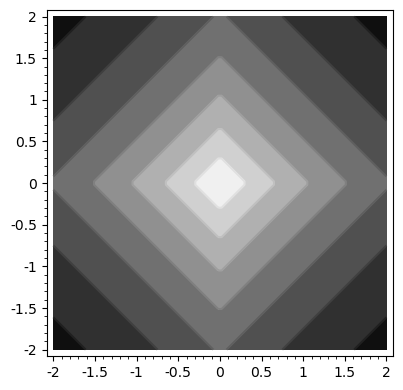}
\end{minipage}
\begin{minipage}{.24\linewidth}
  \includegraphics[width=\linewidth]{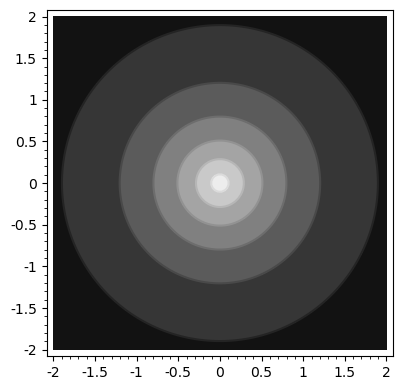}
\end{minipage}
\begin{minipage}{.24\linewidth}
  \includegraphics[width=\linewidth]{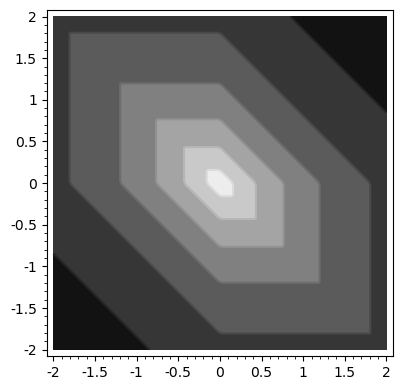}
\end{minipage}
\begin{minipage}{.24\linewidth}
  \includegraphics[width=\linewidth]{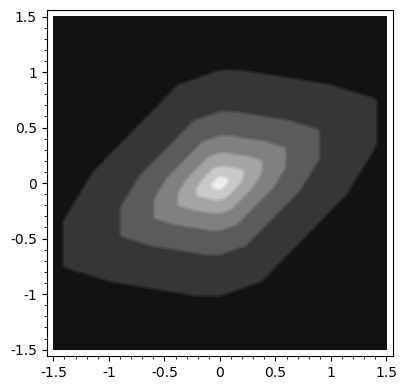}
\end{minipage}
    \caption{Contour plots of the kernels $K_\infty$ under different intensity measures $\Lambda$. As the kernels are translation invariant, the plot shows the contour of the function $x \mapsto K_\infty(0,x)$ for $d = 2$. From left to right: the Laplace kernel obtained from axis-aligned STIT; the exponential kernel obtained from the isotropic STIT; a kernel obtained from a STIT with three cut directions orthogonal to $(1,0)$, $(0,1)$ and $(1,1)$; a kernel obtained from a STIT whose cut directions are ten random Gaussian vectors in $\RR^d$. These plots show that the STIT can approximate many kernel functions.}
    \label{fig:kernel}
\end{figure}

Our second set of main results concerns density estimation and regression using STIT forests. The Mondrian density estimator was first introduced in \cite{BalogTeh2015}. It falls under the more general class of density estimation methods using random space partitions  \cite{WongMa2010, LuJiangWong2013, LiYangWong2016, RamGray2011}. In \cite{mourtada2017universal, mourtada2020minimax}
Mourtada, Ga\"{ï}ffas, and Scornet gave minimax rates for the Mondrian regression estimator. Theorem \ref{thm:regconsistency} and Corollary \ref{cor:densityconsistency} show that the STIT regression and density estimator are both consistent. For the Mondrian, Theorem \ref{thm:density} gives a precise understanding of the differences between the Mondrian forest, the Mondrian kernel, and the Laplace kernel method for density estimation. While the conceptual connection between forests and kernel estimates is well-known \cite{breiman2000some, geurts2006extremely,lin2006random,biau2010layered,biau2016random} and has been noted for the Mondrian case in regression \cite{balog2016mondrian}, explicit formula like that of Theorem \ref{thm:density} have only been obtained for very few random forest models \cite{arlot2014analysis,scornet2016random}, most notably the purely random forest \cite{genuer2012variance}. 
Compared to purely random forests, Mondrian forests and more generally STIT forests are derived from a \emph{self-similar} and \emph{stationary} stochastic process. We crucially exploit these properties to overcome the main hurdle in the analysis of random forests in density estimation, namely, evaluating the volume of a cell in the partition. For the Mondrian case, our analysis uses the precise description of the distribution of a cell in the Mondrian process of Mourtada, Ga\"{i}ffas, and Scornet  \cite{mourtada2020minimax}. This hints that their analysis can be generalized to obtain minimax rates for STIT regression and density estimators. We discuss this open problem, along with a number of other questions in Section \ref{sec:conclusion}.
In summary, our paper points to more fruitful collaborations between stochastic geometry and machine learning. 

\begin{figure}[h!]
    \centering
\begin{minipage}{.45\linewidth}
  \includegraphics[width=\linewidth]{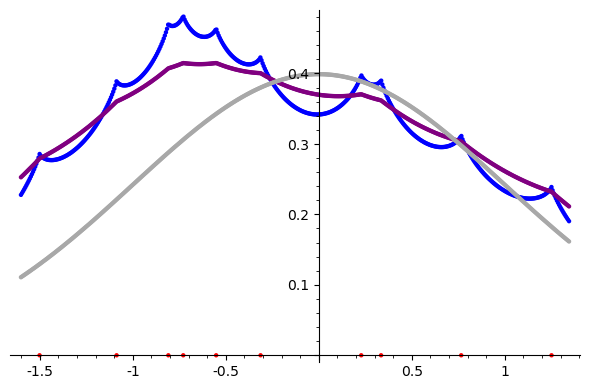}
\end{minipage}
\hspace{.05\linewidth}
\begin{minipage}{.45\linewidth}
  \includegraphics[width=\linewidth]{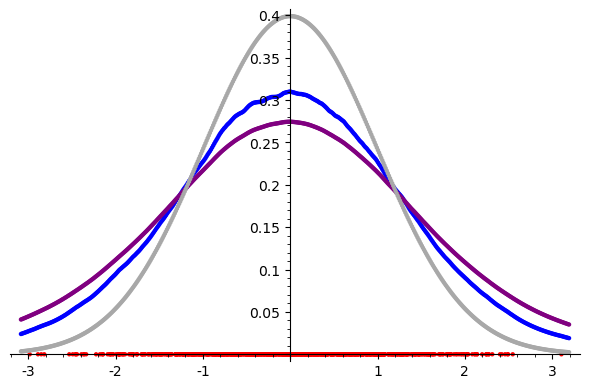}
\end{minipage}
    \caption{Density estimation with Laplace kernel (purple) and the Mondrian forest (blue). The points $X_1, \dots, X_n$ (shown in red on the $x$-axis) are sampled from the standard Gaussian (density shown in gray), with $n = 10$ on the left figure, and $n = 1000$ on the right. The Mondrian method gives more weight to existing data points, and less to those far away. For large number of points, empirically it gives faster convergence to the true density (gray).}
    \label{fig:my_label}
\end{figure}

\section{Stochastic geometry background}\label{sec:background}
The results we present here rely on the theory of stable under iteration (STIT) tessellations. We now summarize relevant results from the stochastic geometry literature that will be used throughout this paper. For a comprehensive text on stochastic geometry, we recommend \cite{weil}. 

\subsection{Stable under iteration tessellations (STIT)}\label{subsec:stit} 
A tessellation of $\RR^d$ is a partition of the entire space into a locally finite collection of convex polytopes that have pairwise disjoint interior. One can think of a tessellation as either the collection of polytopes, called cells, or as the union of their boundaries. We will take the second view in this paper, in which a random tessellation can be considered as a random closed set, see \cite{weil}. For a random tessellation $Y$, we will denote the random collection of its cells by $\mathrm{Cells}(Y)$. 
A random tessellation is stationary if its distribution is translation invariant, and isotropic if its distribution is invariant under rotations about the origin.

Stable under iteration (STIT) tessellations are a class of stationary random tessellations in $\RR^d$ introduced in \cite{Nagel2003}. A rich theory on their properties has since been developed, see \cite{Nagel2008,STITMecke,Thale2011,Thale2013,Thale2013Poisson}.
We will now summarize the construction presented in \cite{Nagel2003}. 
Let $\mathcal{H}^{d}$ denote the space of $d-1$-dimensional affine hyperplanes in $\RR^d$ and let $\Lambda$ be a locally finite and translation-invariant measure on $\mathcal{H}^d$. Assume there exist $d$ hyperplanes with linearly independent normal directions contained in the support of $\Lambda$. 
For $C \subset \RR^d$, define the subset $[C] \subseteq \mathcal{H}^d$ by
\[[C] := \{H \in \mathcal{H}^d : H \cap C \neq \emptyset\}.\]

Assume $\Lambda$ is normalized so that $\Lambda([B(d,1)]) = 1$, where $B(d,1)$ is a ball in $\RR^d$ of unit diameter.
Recall that a hyperplane in $\RR^d$ with normal vector $u \in \RR^d$ and displacement $t \in \RR$ is defined by $H_d(u, t) := \{x \in \RR^d: \langle x, u \rangle = t\}$.
The stationarity and normalization of the measure $\Lambda$ imply that we can write
\begin{align}\label{e:Lambda}
\Lambda(\cdot) = \int_{\mathbb{S}^{d-1}} \int_{\RR} 1_{\{H_d(u,t) \in \cdot\}} \dint t \dint \phi(u) ,
\end{align}
where $\phi$ is a probability measure on the unit sphere $\mathbb{S}^{d-1}$, see Theorem 4.4.1 and (4.30) in \cite{weil}.  

To construct a STIT tessellation with intensity measure $\Lambda$, start from an initial bounded frame $W \subset \RR^d$. Assign to $W$ a random exponential lifetime with parameter
\begin{align*}
 \Lambda([W]) &= \int_{\mathbb{S}^{d-1}} \int_{\RR} 1_{\{H_d(u,t) \cap W \neq \emptyset\}} \dint t \dint \phi(u) 
= \int_{\mathbb{S}^{d-1}} \left(h(W,u) + h(W,-u)\right) \dint \phi(u),   
\end{align*}
where $h(W, u) := \sup_{x \in W} \langle u,x \rangle$ is the support function for $W$. When the lifetime expires, a random hyperplane is generated from the probability measure
\begin{align*}
    \frac{\Lambda(\cdot \cap [W])}{\Lambda([W])} = \frac{1}{\Lambda([W])}\int_{\mathbb{S}^{d-1}} \int_{h(W, -u)}^{h(W,u)} 1_{\{H_d(u,t) \in \cdot \}} \dint t  \dint \phi(u), 
\end{align*}
splitting the window $W$ into two cells $W_1$ and $W_2$.
That is, the random hyperplane splitting $W$
has a random direction $U$ with distribution on $\mathbb{S}^{d-1}$ defined by
\[\dint \Phi(u) := \frac{h(W,u) + h(W,-u)}{\Lambda([W])}\dint\phi(u),\]
and conditioned on $U$, the displacement is uniformly distributed in the interval from $-h(W,-U)$ to $h(W,U)$.

The construction continues recursively and independently in $W_1$ and $W_2$ until some fixed deterministic time $\lambda > 0$. We will denote the random tessellation constructed in $W$ up until time $\lambda$ by $Y_{\Lambda}(\lambda, W)$. The subscript $\Lambda$ will be omitted when the intensity measure is clear from context.
We denote by $Y(\lambda)$ the tessellation on $\RR^d$ such that \begin{align}
  Y(\lambda) \cap W \overset{\mathrm{D}}{=} Y(\lambda, W),  
\end{align}
for all compact sets $W \subset \RR^d$, the existence of which is stated in Theorem 1 of \cite{Nagel2005}. For all $\lambda > 0$, $Y(\lambda)$ is stationary and satisfies the following scaling property:
\[\lambda Y(\lambda) \overset{\mathrm{D}}{=} Y(1).\]
In addition, there is a closed formula for the capacity functional of $Y(\lambda)$. Letting $\mathcal{C}$ denote the set of compact subsets of $\RR^d$, 
the capacity functional of a random closed set $Z$ is defined by
\[T_Z(C) := \PP(Z \cap C \neq \emptyset), \qquad C \in \mathcal{C}.\]
 In particular, for $C \in \mathcal{C}$ with one connected component, it was shown in \cite{Nagel2005} that for a STIT tessellation with intensity measure $\Lambda$,
\begin{align}\label{e:cap}
T_{Y(\lambda)}(C) = 1 - e^{-\lambda\Lambda([C])}.
\end{align}

When $\phi$ is the uniform distribution over the basis vectors $\{e_i\}_{i=1}^d$, that is, when
\begin{align}\label{e:LambdaM}
\Lambda(\cdot) = \frac{1}{d} \sum_{i=1}^d \int_{\RR}  1_{\{H_d(e_i,t) \in \cdot\}} \dint t,
\end{align}
the STIT tessellation construction corresponds to a Mondrian process. 
Another particular case is when $\phi = \sigma$, where $\sigma$ denotes the normalized spherical Lebesgue measure on $\mathbb{S}^{d-1}$. This intensity measure $\Lambda$ and the resulting tessellation are isotropic as well as stationary. This case is studied in \cite{TehRTFs2019}, where they refer to it as the uRTP, or the uniform random tessellation process.

\begin{ex} Let $W = [-0.5,0.5]^2$ and let $\dint \phi(u) = \sum_{i=1}^3 \delta_{u_i}$, where $u_1 = (1,0)$, $u_2 = (0,1)$, and $u_3 = (1/\sqrt{2}, 1/\sqrt{2})$. When the exponential lifetime with parameter $\Lambda([W])$ expires, a cut is made as follows. The direction is $u_i$ with probability $\begin{cases} \frac{1}{2 + \sqrt{2}}, & i = 1,2 \\ \frac{\sqrt{2}}{2 + \sqrt{2}}, & i = 3. \end{cases}$ Then the displacement $t$ is uniform on $[-0.5,0.5]$ if the direction is $u_1$ or $u_2$ and is uniform on $[-\frac{1}{\sqrt{2}}, \frac{1}{\sqrt{2}}]$ if the direction is $u_3$. Figure \ref{fig:example} shows a simulation of this STIT up to lifetime $9$.
\end{ex}

\begin{figure}[h!]
    \centering
\begin{minipage}{.24\linewidth}
  \includegraphics[width=\linewidth]{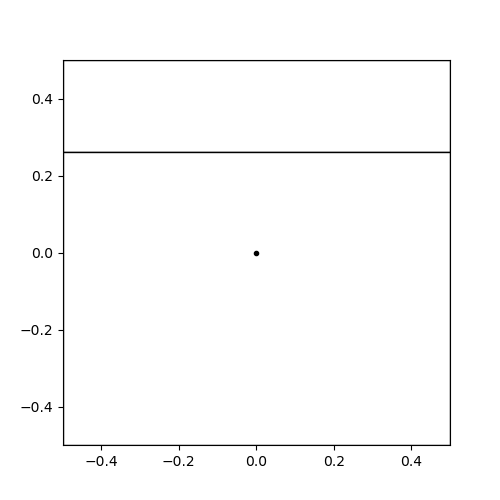}
\end{minipage}
\begin{minipage}{.24\linewidth}
  \includegraphics[width=\linewidth]{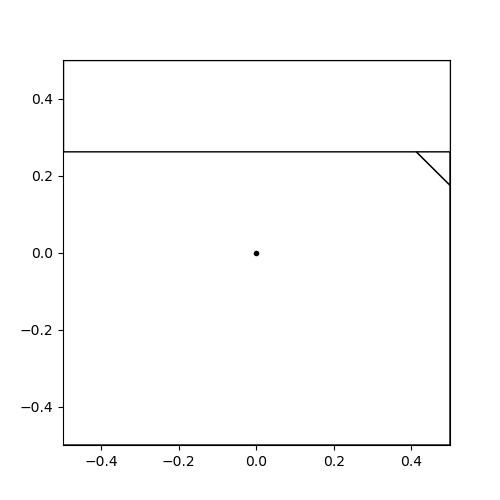}
\end{minipage}
\begin{minipage}{.24\linewidth}
  \includegraphics[width=\linewidth]{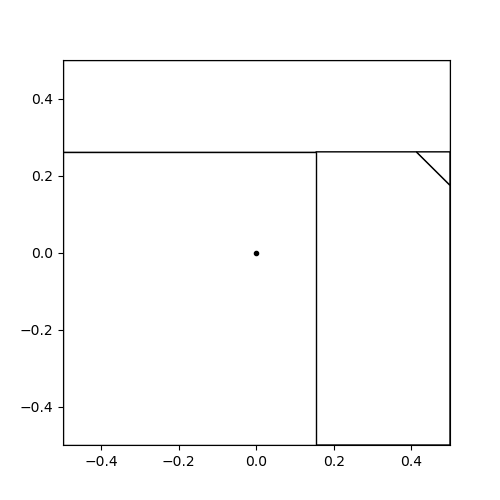}
\end{minipage}
\begin{minipage}{.24\linewidth}
  \includegraphics[width=\linewidth]{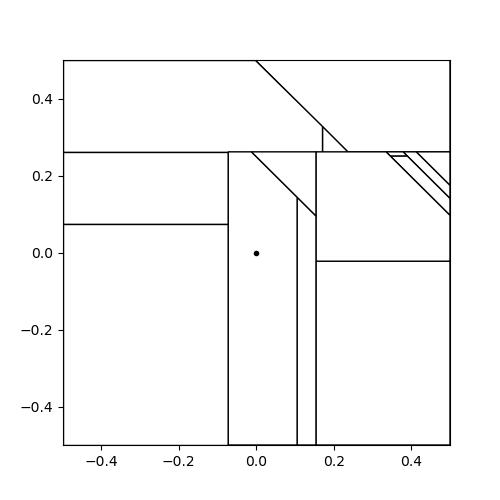}
\end{minipage}
    \caption{An example STIT process with three directions. Each cell $W$ has an independent exponential clock with mean $\Lambda([W])$. When the clock rings, the cell is cut by a hyperplane drawn from the specified distribution $\Lambda$ conditioned to hit this cell. In this simulation, at time $t = 0$ we start with the unit square $W = [-0.5,0.5]^2$, and ran until time $t = 9$, called the lifetime of the STIT. The first three figures show the first three cuts, while the last figure shows the STIT at time $t = 9$, which has 14 cuts.}
    \label{fig:example}
\end{figure}

\subsection{The typical cell of a STIT tessellation}

Let $Y(\lambda)$ denote the stationary STIT tessellation on $\RR^d$ with intensity measure $\Lambda$ and lifetime parameter $\lambda$. Let $c(C)$ be the center of the inball of a convex body $C \subset \RR^d$, and let $\mathcal{C}_0 := \{C \subset \RR^d: C \text{ is a non-empty convex body with } c(C) = 0\}$. Let $V(C)$ denote the $d$-dimensional volume of $C$. The typical cell of $Y(\lambda)$ is the random closed set $Z$ with probability distribution on $\mathcal{C}_0$ given by the following limit (see \cite[(4.8) and (4.9)]{weil}): for any convex body $W \subset \RR^d$ with $V(W) > 0$,
\[\mathbb{Q}(A) := \lim_{r \to \infty} \frac{1}{V(rW)} \EE\sum_{C \in \mathrm{Cells}(Y(\lambda)), C \cap rW \neq \emptyset} 1_{A}(C - c(C)).\]
By Theorem 1 in \cite{Thale2013Poisson}, $Z$ has the same distribution as the typical cell of a Poisson hyperplane tessellation with mean measure $\lambda \Lambda$. By Campbell's theorem (Theorem 3.1.2 in \cite{weil}) and \cite[(4.3)]{weil}, for any nonnegative measurable function $f$ on the space of convex bodies in $\RR^d$,
\begin{align}\label{e:campbell}
    \EE\left[\sum_{C \in \mathrm{Cells}(Y(\lambda))} f(C)\right] = \frac{1}{\EE[V(Z)]}\EE\left[\int_{\RR^d} f(Z + y) \dint y \right].
\end{align}

\subsection{Associated Zonoid}\label{sec:zonoid}

A non-empty convex body $\Pi$ is called a centered zonoid if its support function takes the form $h(\Pi,v) = \int_{\mathbb{S}^{d-1}} |\langle u, v \rangle| \dint \psi(u)$
for some even and finite measure $\psi$ on the unit sphere \cite[p.614]{weil}. 
For a stationary measure $\Lambda$ on $\mathcal{H}^d$ as in \eqref{e:Lambda} and constant $\lambda > 0$, define the zonoid $\Pi$ with support function
\begin{align}\label{eq:hZ}
h(\Pi,v) = \frac{\lambda}{2}\Lambda([[0,v]]) = \frac{\lambda}{2}\int_{\mathbb{S}^{d-1}} |\langle u,v \rangle | \dint \phi(u).
\end{align}
$\Pi$ is called the associated zonoid of the measure $\lambda \Lambda$, or of the stationary STIT with intensity measure $\Lambda$ and lifetime parameter $\lambda$. For each zonoid $\Pi$, there is a unique stationary measure and intensity, and thus a unique STIT tessellation, for which $\Pi$ is the associated zonoid, see \cite[p.156]{weil}.

\section{General cut directions as projections of axis-aligned}

Consider a stationary STIT tessellation $Y$ in $\RR^d$ where the directional distribution of the hyperplanes is the uniform distribution over $n \geq d$ fixed directions $u_1, \ldots, u_n \in \mathbb{S}^{d-1}$. While it is known that the intersection of a STIT tessellation with a lower dimensional linear subspace is distributed as a STIT tessellation \cite{Thale2013, Nagel2005}, to the best of our knowledge, Theorem \ref{thm:projection} below is the first to give an explicit formula that shows \emph{any} STIT with finitely many cut directions can be obtained from the intersection of an \emph{axis-aligned} STIT in a higher dimension with an appropriate subspace. 

\begin{theorem}\label{thm:projection}
Let $u_1, \ldots, u_n$ be $n$ points on the unit sphere $\mathbb{S}^{d-1}$, and define $U : \RR^d \to \RR^n$ to be a $n \times d$ matrix with rows $u_1, \ldots, u_n$. Let $Y$ be a STIT tessellation in $\RR^d$ with associated intensity measure 
\begin{align}\label{e:LambdaU}
\Lambda_U(\cdot) = \frac{1}{n} \sum_{i=1}^n \int_{\RR}  1_{\{H_d(u_i, t) \in \cdot\}} \dint t,
\end{align}
where $H_d(u,t):= \{x \in \RR^d : \langle x, u \rangle = t\}$.
Let $L$ be the linear subspace of $\RR^n$ spanned by the columns of $U$. Then, $U(Y)$ has the same distribution as the intersection of a Mondrian process in $\RR^n$ with the subspace $L$.
\end{theorem}
For the STIT tessellation with intensity measure $\Lambda_U$, a cell is split by a hyperplane supported by $u_i$ with probability proportional to the length of its projection onto the line spanned by $u_i$. When the $u_i$'s are the $d$ standard coordinate vectors, then one obtains an axis-aligned STIT, i.e., the Mondrian process. Theorem \ref{thm:projection} states that using finitely many cut directions has the same effect as performing a linear embedding of the data in higher dimensions and then partitioning it with the Mondrian process. 

\begin{proof}[Proof of Theorem \ref{thm:projection}]
In \cite{Nagel2005}, Lemma 4 and Corollary 1 show that the capacity functional for a STIT tessellation is determined by its associated intensity measure on the space of hyperplanes. Thus it suffices to show that for any Borel set $C \subset \RR^d$,
\[\Lambda_{U}([C]) = \Lambda_{M}([U(C)]),\]
where $\Lambda_M$ is the measure on $\mathcal{H}^d$ associated to the Mondrian process as in \eqref{e:LambdaM}.

Letting $\{e^{(n)}_i\}_{i=1}^n$ denote the standard basis in $\RR^n$, we see that for any $x \in H_d(u_i,t)$,
\[\langle U x, e^{(n)}_i \rangle = \langle x, U^Te_i^{(n)} \rangle = \langle x, u_i \rangle = t .\]
Thus, $U(H_d(u_i,t)) = H_n(e^{(n)}_i, t) \cap L$.
Then, for any Borel set $C \subset \RR^d$,
\begin{align*}
\Lambda_{U}([C]) &= \frac{1}{n} \sum_{i=1}^n \int_{\RR} 1_{\{ H_d(u_i, t) \cap C \neq \emptyset\}} \dint t = \frac{1}{n} \sum_{i=1}^n \int_{\RR} 1_{\{ U(H_d(u_i, t)) \cap U(C) \neq \emptyset\}} \dint t \\
&= \frac{1}{n} \sum_{i=1}^n \int_{\RR} 1_{\{(H_n(e^{(n)}_i, t) \cap L) \cap U(C) \neq \emptyset\}} \dint t \\
&= \frac{1}{n} \sum_{i=1}^n \int_{\RR} 1_{\{H_n(e^{(n)}_i, t) \cap U(C) \neq \emptyset\}} \dint t = \Lambda_{M}([U(C)]).
\end{align*}
This completes the proof.
\end{proof}

\section{STIT Kernel Approximation} 

When one uses the Mondrian process to randomly approximate kernels, there is much to gain from considering more general cuts. 
While the standard Mondrian approximates the Laplace kernel \cite[Proposition 1]{balog2016mondrian}, we show that one can approximate many other kernels with the general STIT. In particular, the exponential kernel $k(x,y) = \exp(-\|x-y\|_2/\sigma)$ can be approximated by the uniform STIT, whose cut directions are chosen uniformly at random from the sphere in $\RR^d$. 

\subsection{Characterization of limiting STIT kernels}
Following \cite{balog2016mondrian}, we define a random feature map corresponding to a STIT tessellation $Y(\lambda)$ as follows. First, define a feature map $\phi$ that maps a location $x \in \RR^d$ to a column vector with a single non-zero entry that indicates the cell of the partition that $x$ lies in. We can then define the kernel
\[K_1(x,y) := \phi(x)^T\phi(y) = \begin{cases} 1, & x, y \text{ in same cell of } Y(\lambda)\\ 0, & \text{otherwise}. \end{cases}\]
Then, let $Y_1, \ldots, Y_M$ be $M$ i.i.d. copies of $Y(\lambda)$. Concatenating the $M$ i.i.d. corresponding random feature maps $\phi^{(1)}, \ldots, \phi^{(M)}$ into a single random feature map and scaling, we define the STIT kernel of order $M$:
\[K_M(x,y) := \frac{1}{M} \sum_{m=1}^M \phi^{(m)}(x)^T\phi^{(m)}(y) = \frac{1}{M}\sum_{m=1}^M 1_{\{y,x \text{ in same cell of } Y_m\}}.\]

\begin{theorem}\label{thm:uniform}
Consider a stationary STIT tessellation $Y(\lambda)$ in $\RR^d$ with associated intensity measure $\Lambda$ and lifetime parameter $\lambda$. 
For each $M$, define $K_M$, the STIT kernel of order $M$, as above. 
Then, 
\begin{align}\label{e:KMlim}
\lim_{M\to \infty} K_M(x,y) = K_{\infty}(x,y) := e^{-\lambda \Lambda([[x,y]])}, \qquad a.s.,
\end{align}
where $[[x,y]]$ denotes the set of hyperplanes hitting the line segment connecting $x$ and $y$ in $\RR^d$. 
\end{theorem}

\begin{proof}[Proof of Theorem \ref{thm:uniform}]
Let $Y_1, \ldots, Y_M$ be i.i.d. realizations of a STIT tessellation $Y(\lambda)$ with stationary intensity measure $\Lambda$. As noted in \cite{balog2016mondrian}, the strong law of large numbers gives the following limit:
\begin{align*}
\lim_{M \to \infty} K_M(x,y) &= \lim_{M \to \infty} \frac{1}{M}\sum_{m=1}^M 1_{\{x, y \text{ in same cell of } Y_m\}} \\
&=  \PP(x, y \text{ in same cell of } Y(\lambda)), \qquad \mathrm{a.s.}
\end{align*}
We shall derive this probability from the capacity functional formula \eqref{e:cap}. 
For $x, y \in \RR^d$, \eqref{e:cap} immediately gives 
\begin{align*}
\PP(x, y \text{ in same cell of } Y(\lambda)) = 
\PP(Y(\lambda) \cap [x,y] = \emptyset) = \exp\left(-\lambda\Lambda([[x,y]])\right).
\end{align*}
\end{proof}

We now give some examples. If $\Lambda(\cdot)$ is isotropic, then
\begin{align}\label{e:Kinfiso}
K_{\infty}(x,y) = e^{-\frac{2 \lambda \kappa_{d-1}}{d\kappa_d}\|x-y\|_2},
\end{align}
where $\kappa_d$ is the volume of the $d$-dimensional unit ball. 
If $\Lambda(\cdot)$ is defined by \eqref{e:LambdaU}, then
\begin{align}\label{e:KinfU}
K_{\infty}(x,y) = e^{-\frac{\lambda}{n}\|U(x-y)\|_1},
\end{align}
where $U$ is the $n \times d$ matrix with rows $u_1, \ldots, u_n$.

To show \eqref{e:Kinfiso} and \eqref{e:KinfU}, note that the hyperplane $H(u_i, t)$ hits the line segment $[x, y]$ if and only if $\langle u_i, x \rangle \leq t \leq \langle u_i, y \rangle$ when $\langle u_i, y - x \rangle \geq 0$ or $\langle u_i, y \rangle \leq t \leq \langle u_i, x \rangle$ when $\langle u_i, y - x \rangle < 0$. Thus, in the isotopic case,
\begin{align*}
\Lambda([[x,y]]) &= \int_{\mathbb{S}^{d-1}}\int_{\RR} 1_{\{H(u,t) \cap [x,y] \neq \emptyset\}} \dint t \dint \sigma(u) = \int_{\mathbb{S}^{d-1}} \int_{0}^{\infty} 1_{\{ t \leq |\langle x-y, u \rangle|\}}  dt \sigma(du) \nonumber  \\
&= \int_{\mathbb{S}^{d-1}} |\langle x-y, u \rangle|  \sigma(du) = \frac{2\kappa_{d-1}}{d\kappa_{d}}\|x-y\|_2,
\end{align*}
and for $\Lambda$ as in \eqref{e:LambdaU},
\begin{align*}
   \Lambda([[x,y]]) &=  \frac{1}{n} \sum_{i=1}^n \int_{\RR} 1_{\{H(u_i,t) \cap [x,y] \neq \emptyset\}} \dint t = \frac{1}{n} \sum_{i=1}^n \int_{0}^{\infty} 1_{\{t \leq |\langle u_i, x - y \rangle|\}} \dint t \\
   &= \frac{1}{n} \sum_{i=1}^n |\langle u_i, x - y\rangle | =  \frac{1}{n} \|U(x- y)\|_1,
\end{align*}
where $U \in \RR^{n \times d}$ is the matrix with rows $u_1, \ldots, u_n$.


The following corollary characterizes all stationary kernels that can be approximated as in Theorem \ref{thm:uniform} in terms of the associated zonoid of the STIT tessellation (see section \ref{sec:zonoid}). In particular, not all positive definite kernels can be approximated this way. The corollary follows directly from \eqref{e:KMlim} and \eqref{eq:hZ}, as well as Theorem 4.2 in \cite{Molchanov2009}, which states that a random vector $\eta$ has a symmetric $\alpha$-stable distribution with $\alpha = 1$ if and only if there is a unique centered zonoid $\Pi$ such that the characteristic function of $\eta$ takes the form $\Phi_{\eta}(v) = e^{-h(\Pi,v)}$, $v \in \RR^d$.

\begin{cor}\label{thm:kernelcharacterization}
Let $K(x,y) = K_0(x-y)$ be a stationary positive definite kernel. There exists a stationary STIT tessellation $Y$ in $\RR^d$ with intensity measure $\Lambda$ and lifetime parameter $\lambda$ such that \eqref{e:KMlim} holds for this kernel if and only if there exists a centered zonoid $\Pi \subset \RR^d$ such that
\[K_0(v) = e^{-\lambda\Lambda([[0,v]])} = e^{-2h(\Pi,v)}, \, v \in \RR^d.\]
In particular, this holds if and only if $K_0(v)$ is the characteristic function of a multivariate Cauchy distribution.
\end{cor}

A larger class of kernels can be approximated by adding additional randomness to the STIT process. The next result characterizes all kernels that can be approximated by random features obtained from STIT tessellations with a random lifetime parameter. 

\begin{theorem}\label{thm:monotone}
Let $\xi$ be a non-negative real-valued random variable. Let $Y$ be a random tessellation such that conditioned on $\xi$, $Y$ is a stationary STIT tessellation with intensity measure $\Lambda$ and lifetime parameter $\xi$. The limiting kernel $K_{\infty}$ based on i.i.d. copies of $Y$ as in Theorem \ref{thm:uniform} is of the form
\[K_{\infty}(x,y) = \Phi(h(\Pi, x-y)),\]
where $\Pi$ is centered zonoid with support function
$h(\Pi,v) = \int_{\mathbb{S}^{d-1}}|\langle u,v \rangle| \dint \phi(u)$,
and $\Phi$ is the Laplace transform of $\xi$.
\end{theorem}

\begin{proof}
Let $\xi$ be a non-negative real-valued random variable and $\Lambda$ a stationary measure on the space of hyperplanes with associated zonoid $\Pi$. Define the random kernel $K_M$ and its limit $K_{\infty}$ as in Theorem \ref{thm:uniform} with respect to the random tessellation generated by first sampling from $\xi$ and then sampling a STIT process with intensity measure $\Lambda$ and lifetime parameter $\xi$. Conditioning on $\xi$ gives 
\[K_{\infty}(v) = \EE[ \EE[K_{\infty}(v) | \xi]] = \EE\left[e^{-\xi\Lambda([[0,v]])}\right] = \EE\left[e^{-\xi \int_{\mathbb{S}^{d-1}} |\langle u, v \rangle| \dint \phi(u)}\right] = \EE\left[e^{-\xi h(\Pi,v)}\right].\]
This is the Laplace transform of the random variable $\xi$ at the value $h(\Pi,v)$.
\end{proof}

For the following corollary, recall that a function $\Phi: [0, \infty) \to \RR$ is completely monotone if $\Phi$ is continuous on $[0,\infty)$, smooth on $(0, \infty)$, and 
\[(-1)^{k} \Phi^{(k)}(r) \geq 0, \quad r > 0, k =0,1,2, \ldots,\] 
where $\Phi^{(k)}$ is the $k$-th order derivative of $\Phi$.

\begin{cor}
Define $Y$ as in Theorem \ref{thm:kernelcharacterization}. Let $K$ be a positive definite stationary kernel of the form 
\[K(x,y) = \Phi(h(\Pi, x-y))\]
for some centered zonoid $\Pi$ and function $\Phi : \RR_{+} \to \RR$. $K$ can be approximated as in \eqref{e:KMlim} with i.i.d. copies of $Y$ for some distribution of $\xi$ if and only if $\Phi$ is completely monotone.
\end{cor}

\begin{proof}

The Hausdorff-Bernstein-Widder Theorem says that $\phi$ is completely monotone if and only if $\phi$ the Laplace transform of a finite
non-negative Borel measure $\mu$ on $[0, \infty)$. Thus, a stationary kernel $K$ can be approximated with the tessellation $Y$ for some $\xi$ if and only if $K(x,y) = \phi(h(\Pi, x-y))$ for a completely monotone function $\phi$.

\end{proof}

Some examples of completely monotone functions $\Phi$ and the corresponding random lifetime parameter $\xi$ needed to approximate a kernel of the form $\Phi(h(\Pi, \cdot))$ include the following:
\begin{enumerate}
    \item If $\xi$ is a gamma random variable $\xi$ with shape parameter $\alpha$ and rate parameter $\lambda$, then 
\[K_{\infty}(x,y) = \left(\frac{\lambda}{\lambda + h(\Pi,v)}\right)^{\alpha}.\]
\item If $\xi$ is uniformly distributed on the interval $[a,b]$ for $0 < a < b < \infty$, then
\[K_{\infty}(x,y) = \frac{e^{-h(\Pi,v)a} - e^{-h(\Pi,v)b}}{h(\Pi,v)(b-a)},\]
\end{enumerate}

\subsection{Convergence rates of STIT kernels}
We next turn to the rate at which the STIT kernel converges over any bounded window $W$. First, consider the case that $K_M$ is generated by $M$ i.i.d. copies of a STIT tessellation $Y(\lambda)$ with lifetime parameter $\lambda$ and intensity measure $\Lambda_U$ as in \eqref{e:LambdaU}. Proposition 2 in \cite{balog2016mondrian} states that for the Mondrian kernel $\tilde{K}_M$ on $\RR^n$ with uniform measure $\Lambda$ in all axis-aligned directions and lifetime $\lambda n$,
\begin{align}\label{eqn:balog}
&\PP\left(\sup_{x,y \in W} |\tilde{K}_M(x,y) - \tilde{K}_{\infty}(x,y)| \geq \delta\right)  \\
& \qquad \qquad \leq \left(2^{1/(2n)}4M^2\lambda^2|\mathcal{X}|^2e^{2\lambda |\mathcal{X}|}/n\right)^{1/(3 + 1/2n)}e^{-M \delta^2/(12n+2)}. \nonumber
\end{align}
We can then apply Theorem \ref{thm:projection} to reduce the STIT $Y(\lambda)$ with intensity measure $\Lambda_U$ to the projection of a Mondrian in $\RR^n$, and use \eqref{eqn:balog} to obtain a uniform convergence rate:
\begin{align*}
    &\PP\left(\sup_{x,y \in W} |K_M(x,y) - K_{\infty}(x,y)| \geq \delta\right) \\
    &\leq \PP\left(\sup_{z, w \in U(W)} |\tilde{K}_{M}(z,w) - e^{-\frac{\lambda}{n}\|z- w\|_1}| \geq \delta\right) \\ 
    &\leq \left(2^{1/(2n)}4\lambda^2M^2|\mathcal{X}|^2e^{2\lambda|\mathcal{X}|/n}/n^3\right)^{1/(3 + 1/2n)}e^{-M \delta^2/(12n+2)},
\end{align*}
where $|\mathcal{X}|$ is the sum of the widths in each direction of a box in $\RR^n$ that $U(W)$ is contained in. 
Note that this rate worsens as the number of directions $n$ increases. The argument in \cite{balog2016mondrian} relies on a discretization of $W$ and then an application of Hoeffding's inequality. Since all cells of the Mondrian are axis-aligned boxes, one can ensure sufficient discretization in each direction by covering $W$ with an $\ee$-grid as in \cite{balog2016mondrian}. Adding possible directions for the cuts means that applying their argument without lifting would give a similarly slow rate. For general $\Lambda$, we cannot rely on matching the geometry of the discretization and the cells, since the shape of the cells can vary widely. In particular, if $\Lambda$ is highly concentrated in a certain direction, partition cells likely have very small width in this direction compared to other directions. 

It is to be expected then, that some measure of the shape of the cells shows up in the general analysis. Indeed, let $\Pi$ be the associated zonoid of the intensity measure $\Lambda$. 
The shape of $\Pi$ reflects the spherical distribution and is connected to the average shape of the cells. In particular, define \[h_{\min}(\Pi) := \min_{u \in \mathbb{S}^{d-1}} h(\Pi, u), \text{   and   } h_{\max}(\Pi) := \max_{u \in \mathbb{S}^{d-1}} h(\Pi, u).\]
The following theorem shows that we can generalize the exponential decay of the uniform error over $W$ in $M$ from \cite{balog2016mondrian} to the general STIT kernel, but with a rate depending on the ratio $h_{\min}(\Pi)/h_{\max}(\Pi)$, and at the cost of an additional polynomial term in $M$.
\begin{theorem}\label{thm:rate}
Let $K_M$ be the STIT kernel defined as in Theorem \ref{thm:uniform} associated to a STIT $Y$ with intensity measure $\Lambda$ and lifetime parameter $\lambda$. Let $W$ be a compact window in $\RR^d$ and $\Pi$ be the associated zonoid of $\Lambda$. Then, for any $\delta > 0$ and fixed $\tau \in (0,1)$,
\begin{align*}
    \PP\left(\sup_{x,y \in W} |K_M(x,y) - K_{\infty}(x,y)| \geq \delta\right) \leq  C M^{d + \frac{d}{2d+1}}\exp^{-M\min\left\{ \frac{\delta^2}{4d+2}, \frac{\delta \tau h_{\min}(\Pi)}{8h_{\max}(\Pi))}\right\}},
\end{align*}
where $C$ depends on $\delta$, $W$, $\tau$, and $d$.
\end{theorem}
\begin{proof}
Fix $r > 0$ such that $W \subseteq B_r$, where $B_r$ denotes a ball in $\RR^d$ centered at the origin with radius $r$, and let $\ee > 0$. Take $\mathcal{U}$ to be an $\ee$-covering of $B_{R + \ee}$ for some $R > r$. 
Denote by $\tilde{Y}^{(M)} = \cup_{i=1}^M Y_i$ the random partition of $\RR^d$ obtained by superimposing all $M$ STIT partitions used to build $K_M$. Note that $\tilde{Y}^{(M)}$ is a stationary random tessellation since it is the superposition of independent stationary random tessellations. 
Define the following ``bad" events:
\begin{itemize}
\item $A_1 := \{\exists \text{ a cell of } \tilde{Y}^{(M)} \text{ with center outside } B_R \text{ that intersects } B_r\}$
\item $A_2 := \{\exists\text{ a cell of } \tilde{Y}^{(M)} \text{ with center in } B_R \text{ without a point of }\mathcal{U}\}$
\item $A_3 := \{\exists\text{ a cell of } \tilde{Y}^{(M)} \text{ with center in } B_R \text{ with diameter } > \frac{\delta}{8\lambda h_{\max}(\Pi)}\}$
\item $A_4 := \{\exists u_1, u_2 \in \mathcal{U}\text{ such that }|K_m(u_i,u_j) - K_{\infty}(u_i, u_j)| > \frac{\delta}{2}\}$
\end{itemize}
\begin{figure}
    \centering
    \includegraphics[width=.5\textwidth]{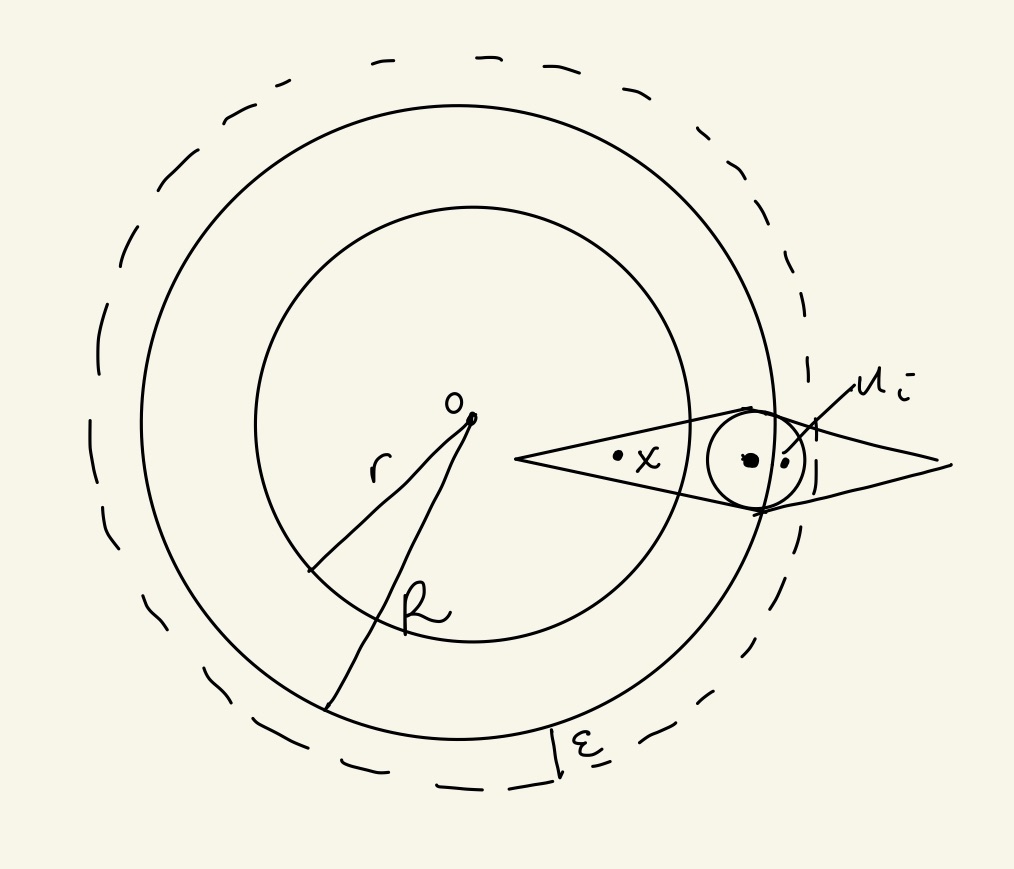}
    \caption{Illustration of partition cell on boundary in the ``good" event $(A_1 \cup A_2 \cup A_3 \cup A_4)^c$.}
    \label{fig:STITcell}
\end{figure}

Suppose the event $A_1^{c} \cap A_2^{c} \cap A_3^{c} \cap A_4^c$ holds. For each $x \in B_r$, this implies that the cell of $\tilde{Y}^{(M)}$ containing $x$ has center in $B_R$, its inball is contained in $B_{R + \ee}$, and it contains a point $u_i \in \mathcal{U}$ such that $\|u_i - x\| \leq \frac{\delta}{8\lambda h_{\max}(\Pi)}$, see Figure \ref{fig:STITcell}. 
By Corollary \ref{thm:kernelcharacterization}, $K_{\infty}(x,y) = e^{-2\lambda h(\Pi,x-y)}$. For $x,y \in B_r$, there are $u_i$ and $u_j \in \mathcal{U}$ such that
\begin{align*}
|K_{\infty}(u_i,u_j) - K_{\infty}(x, y)| &\leq 2\lambda|h(\Pi, u_i - u_j) - h(\Pi, x-y)|\\ &\leq 2\lambda h_{\max}(\Pi) \|(u_i - u_j) - (x-y)\| \\
&\leq 2\lambda h_{\max}(\Pi)\left(\|u_i - x\| + \|u_j - y\|\right) \leq \frac{\delta}{2},
\end{align*}
where we have used the inequality $|e^{-|x|} - e^{-|y|}| \leq |x-y|$ and the subadditivity of the support function.
Then, by the assumption $A_4^{c}$ and the fact that $K_M$ is constant over each cell of $\tilde{Y}^{(M)}$,
\begin{align*}
    &|K_M(x,y) - K_{\infty}(x,y)| \\
    &\leq |K_M(x,y) - K_M(u_i, u_j)| + |K_M(u_i,u_j) - K_{\infty}(u_i, u_j)| + |K_{\infty}(u_i,u_j) - K_{\infty}(x, y)| \\
    &= |K_m(u_i,u_j) - K_{\infty}(u_i, u_j)| + |K_{\infty}(u_i,u_j) - K_{\infty}(x, y)| \leq \delta.
\end{align*}
This implies
\begin{align*}
    \PP\left(\sup_{x,y \in W} |K_M(x,y) - K_{\infty}(x,y)| \geq \delta\right) 
    &\leq \PP(A_1 \cup A_2 \cup A_3 \cup A_4).
\end{align*}
Thus, it remains to prove an upper bound on the probabilities of events $A_i$. 

We first make the following claim: the typical cell of $\tilde{Y}^{(M)}$, denoted by $\tilde{Z}$, has the same distribution as 
the typical cell of a Poisson hyperplane tessellation with mean measure $\lambda M\Lambda(\cdot)$. Indeed, let $Z_{0,i}$ be the cell of $Y_i$ containing the origin for each $i$ and let $\tilde{Z}$ be the zero cell of $\tilde{Y}^{(M)}$. Then, note that
\begin{align*}
  \tilde{Z}_0 = \cap_{i=1}^M  Z_{0,i} = Z_{0,1} \cap (Z_{0,1} \cap Z_{0,2}) \cap \cdots (Z_{0,1} \cap \cdots \cap Z_{0,M}),
\end{align*}
which implies that $\tilde{Z}_0$ has the same distribution as the zero cell of $Y(\lambda M)$. The claim then follows from Theorem 1 in \cite{Thale2013Poisson}, and the fact that the zero cell of a stationary tessellation determines the distribution of the typical cell (Theorem 10.4.1 in \cite{weil}). 
In particular, this implies $\EE[V(\tilde{Z})] = (M\lambda)^{-d}V(\Pi)^{-1}$ (see  \cite[(10.44)]{weil}). 

Now, to bound the probability of $A_1$, Markov's inequality and \eqref{e:campbell} give
\begin{align*}
    \PP(A_1) &\leq \EE\left[ \sum_{C \in \text{Cells}(\tilde{Y}^{(M)})} 1_{\{c(C) \notin B_R\}} 1_{\{C \cap B_r \neq \emptyset\}}\right] \\
    & = \EE[V(\tilde{Z})]^{-1} \EE\left[\int_{\RR^d} 1_{\{y \notin B_R\}}1_{\{y + \tilde{Z} \cap B_r\}} \dint y \right].
\end{align*}
Note that $y + \tilde{Z} \cap B_r \neq \emptyset$ if the line segment $[y, r\frac{y}{\|y\|}]$ is contained in $y + \tilde{Z}$. Then,
\begin{align*}
 \PP(A_1)
    &\leq \EE[V(\tilde{Z})]^{-1} \EE\left[\int_{\RR^d} 1_{\{y \notin B_R\}} 1_{\{[y, r\frac{y}{\|y\|}] \subseteq \tilde{Z} + y\}} \dint y \right] \\
    &= (\lambda M)^dV(\Pi) \int_{\mathbb{S}^{d-1}}\int_{R}^{\infty} t^{d-1} \PP([0, (r-t) u] \subseteq \tilde{Z}) \dint t \dint u.
\end{align*}
By Corollary 10.4.1 in \cite{weil} and our claim about the distribution of $\tilde{Z}_0$,
\begin{align*}
 \PP([0, (r-t) u] \subseteq \tilde{Z}) \leq \PP([0, (r-t) u] \subseteq \tilde{Z}_0) = e^{-\lambda M(t-r)h(\Pi, u)},   
\end{align*}
Then,
\begin{align*}
     \PP(A_1) &\leq (\lambda M)^dV(\Pi) \int_{\mathbb{S}^{d-1}}\int_{R}^{\infty} t^{d-1} e^{-\lambda M(t-r)h(\Pi, u)} \dint t \dint u \\
     &\leq (\lambda M)^dV(\Pi) d\kappa_d \int_{R}^{\infty} t^{d-1} e^{-\lambda M (t - r)h_{\min}(\Pi)} \dint t  \\
     &= (\lambda M)^dV(\Pi)d \kappa_d (\lambda M h_{\min}(\Pi))^{-d}e^{\lambda M h_{\min}(\Pi)r}\int_{\lambda M h_{\min}(\Pi) R}^{\infty}x^{d-1}e^{-x} \dint x \\
     &\leq V(\Pi)d \kappa_d h_{\min}(\Pi)^{-d}e^{\lambda M h_{\min}(\Pi)r}\Gamma(d)\left[1 - \left(1 - e^{-\frac{\lambda M h_{\min}(\Pi)R}{\Gamma(d + 1)^{1/d}}}\right)^d\right] \\
     &\leq V(\Pi)d \kappa_d h_{\min}(\Pi)^{-d}\Gamma(d+1)e^{-\lambda M h_{\min}(\Pi)\left(R\Gamma(d + 1)^{-1/d} - r\right)}.
\end{align*}
where the second to last inequality follows from \cite[Theorem 1]{Alzer1997} and the last inequality from the fact that $(1-e^{-x})d \geq 1 - de^{-x}$ for $x > 0$. 

To bound the probability of $A_2$, note that $A_2$ implies there is a cell of $\tilde{Y}^{(M)}$ with center in $B_R$ such that the largest ball contained in the cell has radius smaller than $\ee$. That is, there is a cell with inradius smaller that $\ee$. Indeed, if all cells had this property, then all cells would contain a point of $\mathcal{U}$. By Markov's inequality, \eqref{e:campbell}, and Theorem 10.4.8 in \cite{weil},
\begin{align*}
    \PP(A_2) &\leq \EE\left[\sum_{C \in \mathrm{Cells}(\tilde{Y}^{(M)})} 1_{\{r(C) \leq \ee\}}1_{\{c(C) \in B_R\}}\right] \\
    &= \EE[V(\tilde{Z})]^{-1}\EE\left[\int_{\RR^d} 1_{\{r(\tilde{Z} + y) \leq \ee\}}1_{\{y \in B_R\}} \dint y\right] \\
    &= (\lambda M)^dV(\Pi)V(B_R)\PP(r(\tilde{Z}) \leq \ee) \\
    &\leq (\lambda M)^dV(\Pi)V(B_R)(1 - e^{-2M\lambda \ee}).
\end{align*}

Also by Markov's inequality and \eqref{e:campbell},
\begin{align*}
    \PP(A_3) &\leq \EE\left[\sum_{C \in \mathrm{cells}(\tilde{Y}^{(M)})} 1_{\{\text{Diam}(C) \geq \frac{\delta}{8\lambda h(\Pi, u_{\max})} \}}1_{\{c(C) \in B_R\}}\right] \\
    &= \EE[V(\tilde{Z})]^{-1}V(B_R)\PP\left(\text{Diam}(\tilde{Z}) \geq \frac{\delta}{8\lambda h_{\max}(\Pi)} \right).
\end{align*}
Then, by the proof of Theorem 2 in section 8 of \cite{HugSchneider2007}, for any fixed $\tau \in (0,1)$, 
\begin{align}\label{e:diameter}
    \PP\left(\text{Diam}(\tilde{Z}) \geq \frac{\delta}{8\lambda h_{\max}(\Pi)} \right) \leq c_0 e^{- M \frac{\tau\delta h_{\min}(\Pi)}{8h_{\max}(\Pi)} },
\end{align}
where $c_0 = c_0(\tau, \Pi, \delta)$ and thus
\begin{align*}
    \PP(A_3) &\leq c_0(\lambda M)^dV(\Pi)V(B_R) e^{-M \frac{\delta \tau h_{\min}(\Pi)}{8h_{\max}(\Pi)}}.
\end{align*}

Finally, to bound the probability of $A_4$, one can choose a $\ee$-covering $\mathcal{U}$ of size at most $ \left(\frac{6(R + \ee)}{\ee}\right)^d$, see Lemma 4.1 in \cite{PollardBook}. A union bound and Hoeffding's inequality gives
\begin{align*}
    \PP(A_4) &\leq \sum_{i,j \in \mathcal{U}} \PP\left(|K_M(u_i,u_j) - K_{\infty}(u_i, u_j)| > \frac{\delta}{2}\right) = \left(\frac{6(R+\ee)}{\ee}\right)^{2d} 2e^{- M \delta^2/2}.
\end{align*}
For all $\ee > 0$ small enough, putting these bounds together and using the bound $1 - e^{-x} \leq x$, we obtain, for all $M$ large enough,
\begin{align*}
    &\PP(A_1 \cup A_2 \cup A_3 \cup A_4) \leq \sum_{i=1}^4 \PP(A_i) \\
    &\leq c_1M^{d-1}R^{d-1}\left(e^{-\lambda M h_{\min}(\Pi)(R\Gamma(d + 1)^{-1/d} - r)} +  M R e^{- M \frac{\delta \tau h_{\min}(\Pi)}{8h_{\max}(\Pi)}}\right) \\
    &\qquad \qquad +  c_2 R^d M^{d + 1} \ee + c_3 \ee^{-2d}e^{- M \delta^2/2},
\end{align*}
where the constants $c_i$ are all positive and depend on $\delta$, $\tau$, $\lambda$, $d$, $r$, and $\Pi$.
Differentiating with respect to $\ee$ and setting equal to zero, the above upper bound hits a minimum at
\begin{align*}
    \ee_0 = c_4 M^{-\frac{d + 1}{2d + 1}}R^{-\frac{d}{2d + 1}}e^{-\frac{M\delta^2}{4d + 2}}.
\end{align*}
Thus, since the above bound holds for any $\ee > 0$,
\begin{align*}
    &\PP(A_1 \cup A_2 \cup A_3 \cup A_4) \\
    &\leq c_1M^{d-1}R^{d-1}\left(e^{-\lambda M h_{\min}(\Pi)(R\Gamma(d + 1)^{-1/d} - r)} +  M R e^{- M \frac{\delta \tau h_{\min}(\Pi)}{8h_{\max}}}\right) \\
    & \qquad \qquad +  c_5 R^{d-\frac{d}{2d + 1}} M^{d + 1 -\frac{d + 1}{2d + 1}}e^{-\frac{M\delta^2}{4d + 2}} + c_6 M^{\frac{2d(d+1)}{2d + 1}}R^{\frac{2d^2}{2d + 1}}e^{\frac{2dM\delta^2}{4d + 2}- \frac{M \delta^2}{2}} \\
    &\leq c_1 M^{d-1} R^{d-1} e^{-\lambda M h_{\min}(\Pi)(R\Gamma(d + 1)^{-1/d} - r)} + c_1 M^d R^d e^{- M \frac{\delta \tau h_{\min}(\Pi)}{8 h_{\max}(\Pi)}} \\
    &\qquad \qquad  + c_{7} R^{\frac{2d^2}{2d + 1}} M^{d  + \frac{d}{2d + 1}}e^{-\frac{M\delta^2}{4d + 2}}.
\end{align*}
Finally, letting $R = \Gamma(d + 1)^{1/d}\left(r + \frac{\delta^2}{\lambda h_{\min}(\Pi)(4d + 2)}\right)$ gives: for any fixed $\tau \in (0,1)$
\begin{align*}
    \PP(A_1 \cup A_2 \cup A_3 \cup A_4) &\leq c_8 M^{d + \frac{d}{2d+1}}e^{-M \min\left\{\frac{\delta \tau h_{\min}(\Pi)}{8h_{\max}(\Pi)}, \frac{\delta^2}{4d + 2}\right\}}.
\end{align*}

\end{proof}

\begin{rem}
 While Theorem \ref{thm:rate} gives the rate of convergence with respect to the number of STIT partitions $M$ in the approximation, the constant $C$ in the upper bound that depends on $\delta$, $W$, $\tau$, and $d$ cannot be written explicitly from our proof. This is due to the constant in the tail bound of the diameter of the typical cell \eqref{e:diameter}. We leave for future work a more explicit and optimal tail bound for the diameter, see open questions in section \ref{sec:conclusion}.
\end{rem}

\section{Density estimation with STIT forests}\label{sec:density}

Let $\lambda > 0$ and consider a STIT tessellation $Y(\lambda d)$ with associated intensity measure $\Lambda$ and denote by $Z^{\lambda}_x$ the cell of $Y(\lambda d)$ containing $x \in \RR^d$. We consider lifetimes scaling with dimension $d$ so as to better compare with Laplace kernel density estimation. Indeed, Theorem \ref{thm:uniform} shows that sampling from $Y(d)$ approximates the Laplace kernel when $\Lambda$ is as in \eqref{e:LambdaM}.

Let $X_1, \ldots, X_n$ be $n$ i.i.d. points drawn from a distribution with unknown density $f$. 
Define the density estimator for $f$ from a single STIT tree as 
\[f_{\lambda, n}(x) :=  \frac{1}{n} \sum_{i=1}^n \frac{1_{\{x \in Z^{\lambda}_{X_i}\}}}{V(Z^{\lambda}_{X_i})} = \frac{1}{n} \sum_{i=1}^n \frac{1_{\{X_i \in Z^{\lambda}_x\}}}{V(Z^{\lambda}_{x})},\]
where $V(Z)$ denotes the volume of a cell $Z$.
Dividing by the volumes ensures that $f_n(x)$ integrates to one, that is, it is indeed a density. Similarly, define the density estimator for $f$ from a STIT forest with $M$ i.i.d. trees as
\begin{align}\label{e:ftree}
f_{\lambda, n, M}(x) := \frac{1}{M}\sum_{m=1}^M f_n^{(m)}(x),
\end{align}
where $\{f_{\lambda, n}^{(m)}\}_{m=1}^M$ are i.i.d. copies of the estimator $f_{\lambda,n}$ obtained from $M$ i.i.d. copies of the STIT tessellation $Y_1, \ldots, Y_M$. By the strong law of large numbers, as $M \to \infty$,
\[f_{\lambda, n,M}(x) \to f_{\lambda, n,\infty}(x) := \EE\left[f_n(x)\right], \qquad a.s.\]

We first show that the ideal STIT forest estimator $f_{\lambda, n, \infty}$ is always a kernel estimator for any stationary measure $\Lambda$. Then, in our last main result (cf. Theorem \ref{thm:density}) we will derive the explicit formula for the Mondrian forest.

\begin{prop}\label{prop:stit.is.kernel}
In the above settings,
\begin{align}\label{e:finf_ker}
f_{\lambda, n, \infty}(x) 
= \frac{\lambda^d}{n}\sum_{i=1}^n K_{0,\Lambda}(\lambda(x - X_i)),
\end{align}
where
\[K_{0,\Lambda}(x) := \EE\left[\frac{1_{\{x \in Z_0\}}}{V(Z_0)}\right],\]
and $Z_0$ is the zero cell of $Y(d)$, i.e. the cell of $Y(d)$ containing the origin.
That is, the ideal STIT forest density estimator (with infinitely many trees) is a kernel density estimator with kernel $K_{0,\Lambda}$ and bandwidth $\lambda^{-1}$. 
\end{prop}
\begin{proof}
First, note that $K_{0, \Lambda}$ is symmetric and $\int K_{0,\Lambda}(x) \dint x = 1$. Also, for $\lambda > 0$, define
\[K_{0,\Lambda}^{\lambda}(x) := \EE\left[\frac{1_{\{x \in Z^{\lambda }_0\}}}{V(Z^{\lambda }_0)}\right] = \EE\left[\frac{1_{\{\lambda x \in Z_0\}}}{\lambda^{-d} V(Z_0)}\right] = \lambda^d K_{0,\Lambda}(\lambda x),\]
where the equalities follow from the self-similarity property of the STIT tessellation.
By stationarity, for all $x,y \in \RR^d$, $\EE\left[\frac{1_{\{x \in Z_y\}}}{V(Z_y)}\right] = \EE\left[\frac{1_{\{x - y \in Z_0\}}}{V(Z_0)}\right]$. Thus,
\[f_{\lambda, n, \infty}(x) = \EE[f_{\lambda, n}(x)] = \frac{1}{n} \sum_{i=1}^n \EE\left[\frac{1_{\{X_i \in Z_{x}^{\lambda}\}}}{V(Z_x^{\lambda})}\right] = K^{\lambda}_{0, \Lambda}(x - X_i),\]
and we obtain \eqref{e:finf_ker}. 
\end{proof}

One may guess that this limiting kernel $K^\lambda_0$ is just the STIT kernel $K_\infty$ of Theorem \ref{thm:kernelcharacterization}. However, already for the Mondrian case, this is not true. This difference was discussed in \cite{balog2016mondrian} in the context of regression. Our next main result, Theorem \ref{thm:density}, makes precise the limiting behavior of the Mondrian forest density estimator. In particular, the limiting density estimator $f_{\lambda, n, \infty}$ of the Mondrian forest is revealed to be a kernel estimator with a Laplace component and a volume correction term. 

\begin{theorem}\label{thm:density}
Consider $n$ i.i.d. points $X_1, \ldots, X_n$ drawn from a distribution with unknown density $f$. Let $\lambda > 0$, and let $f_{\lambda, n,M}$ denote the random forest estimator as in \eqref{e:ftree} obtained from running $M$ independent Mondrian processes, i.e. STIT processes with $\Lambda$ as in \eqref{e:LambdaM}.
Let $f_{\lambda, n,\infty}$ denote the ideal forest estimator (obtained with infinitely many trees) in this case. Then,
\begin{align*}
f_{\lambda, n, \infty}(x) = \frac{\lambda^d}{n} \sum_{i=1}^n e^{-\lambda \|x - X_i\|_1} \prod_{j=1}^d h\left(\lambda|x_j - X_i^{(j)}|\right),
\end{align*}
where $h(t) := 1 -  t e^{t} E_1(t)$
and $E_1(\cdot)$ is the exponential integral $E_1(t) := \int_{t}^{\infty} \frac{e^{-s}}{s} \dint s$.
\end{theorem}
\begin{cor}
In the setting of Theorem \ref{thm:density}, the density estimator $f_{\lambda, n,M}$ is a random approximation of the kernel density estimator $f_{\lambda, n,\infty}$ corresponding to a bandwidth $\lambda^{-1}$ and kernel
\begin{equation}\label{eqn:K0}
K_0(x) =  e^{- \|x\|_1} \prod_{j=1}^d h(|x_j|).
\end{equation}
\end{cor}
Note that the first term $e^{- \|x\|_1}$ of $K_0$ is the Laplace kernel, and the second term decays as $x$ gets farther from the origin. 
In other words, the kernel of the infinite Mondrian forest decays faster from the observed points compared to the Laplace kernel. This can be seen clearly in Figure \ref{fig:my_label}. Recall that the Laplace kernel density estimator for a density $f$ given i.i.d. draws $X_1, \ldots, X_n$ and bandwidth $\lambda^{-1}$ is given by 
\[\hat{f}_{\lambda, n}(x) = \frac{1}{n}\left(\frac{\lambda}{2}\right)^d \sum_{i=1}^n e^{- \lambda \|x - X_i\|_1}.\]
Comparing this estimator with $f_{\lambda, n, \infty}$ shows that $f_{\lambda, n, \infty}$ is a more data-biased version of the Laplace kernel, as discussed in the introduction. Starting from $M$ samples of the STIT tessellation $Y(\lambda d)$, the following density estimator for a density $f$ can also be considered. For $x \in \RR^d$, define
\[\hat{f}_{\lambda, n,M}(x) = \frac{1}{n} \sum_{i=1}^n \frac{\frac{1}{M}\sum_{m=1}^M 1_{\{X_i \in Z^{\lambda}_x\}}}{\frac{1}{M}\sum_{m=1}^M V(Z^{\lambda}_x)}.\]
This estimator is a Monte Carlo estimate for the the Laplace kernel density estimator. 

\begin{proof}[Proof of Theorem \ref{thm:density}]
Following from Proposition \ref{prop:stit.is.kernel}, it remains to compute $K_{0,\Lambda}$ explicitly for the Mondrian process. In this case, we will denote the kernel by $K_0$. 
By \cite[Proposition 1]{mourtada2020minimax}, the zero cell $Z_0$ of the Mondrian process $Y(d)$ has the distribution 
\begin{align}\label{e:Z0}
Z^{\lambda}_0 \overset{\mathrm{D}}{=} \prod_{j=1}^d [-T_{j,0}, T_{j,1}]
\end{align}
where $\{T_{j,k}\}_{j=1, \ldots, d; k=0,1}$ are i.i.d. exponential random variables with parameter $1$. Then,
\begin{align*}
K_0(x) = \EE\left[\frac{1_{\{x \in Z_{0}\}}}{V(Z_{0})}\right] &= \prod_{j=1}^d \left(\EE\left[\frac{1_{\{0 \leq x_j \leq T_{j,1}\}}}{T_{j,0} + T_{j,1}}\right] + \EE\left[\frac{1_{\{-T_{j,0}  \leq x_j \leq 0\}}}{T_{j,0} + T_{j,1}}\right]\right) \\
&= \prod_{j=1}^d \EE\left[\frac{1_{\{|x_j| \leq T_{0}\}}}{T_{0} + T_{1}}\right].
\end{align*}
By the change of variable $z = t + s$,
\begin{align*}
&\EE\left[\frac{1_{\{|x| \leq T_0\}}}{T_0 + T_1}\right] = \int_{0}^{\infty} \int_0^{\infty} 1_{\{|x| \leq t\}} \frac{e^{-(t + s)} }{t + s} \dint t \dint s = \int_{0}^{\infty} \int_{ s}^{\infty} 1_{\{|x| \leq z - s\}} \frac{e^{- z}}{z}  \dint z \dint s \\
&= \int_{0}^{\infty}\frac{e^{-z}}{z}  \left( \int_0^{\infty} 1_{\{s \leq z\}}1_{\{s \leq z - |x|\}} \dint s \right) \dint z = \int_{0}^{\infty}1_{\{z \geq |x|\}}\left(\frac{z}{z} - \frac{|x|}{z}\right) e^{-z} \dint z \\
&= \left(\int_{|x|}^{\infty} e^{- z} \dint z  - \int_{|x|}^{\infty}\frac{|x|}{z} e^{- z} \dint z \right) = e^{-|x|} - |x| E_1(|x|),
\end{align*}
where $E_1(\cdot)$ is the exponential integral $E_1(t) = \int_{t}^{\infty} \frac{e^{-s}}{s} \dint s$. 
Then, letting $h(t) := 1 -  t e^{t} E_1(t)$,
\begin{align*}
\EE\left[\frac{1_{\{x \in Z_{0}\}}}{V(Z_{0})}\right] &= \prod_{j=1}^d \EE\left[\frac{1_{\{|x_j| \leq T_{0}\}}}{T_{0} + T_{1}}\right] 
= \prod_{j=1}^d \left(e^{-|x_j|} - |x_j| E_1(|x_j|) \right) = e^{-\|x\|_1}\prod_{j=1}^d h(|x_j|).
\end{align*}
Finally, we have
\begin{align*}
f_{\lambda, n, \infty}(x) &= \EE[f_{\lambda, n}(x)] = \frac{\lambda^d}{n} \sum_{i=1}^n \EE\left[\frac{1_{\{\lambda(X_i- x) \in Z_{0}\}}}{V(Z_{0})}\right] \\
&=  \frac{\lambda^d}{n} \sum_{i=1}^n e^{-\lambda \|x - X_i\|_1} \prod_{j=1}^d h\left(\lambda |x_j -X^{(j)}_i|\right),
\end{align*}
where $X_i = (X^{(1)}_i, \ldots, X^{(d)}_i)$ for each $i=1, \ldots, n$. This is indeed a density, since for each $j$, $\int_{\RR} e^{- |x_j|} \dint x_j = 2$ and $\int_{\RR}|x_j|E_1(|x_j|) \dint x_j = 1$.
\end{proof}

\begin{rem}\label{rem:more.general}
The proof of Theorem \ref{thm:density} can be used as a template to make similar statements for an arbitrary STIT. However, a key missing ingredient is knowledge about the distribution of the volume of the zero cell $Z_0^{\lambda}$. Theorem 1 in \cite{Thale2013Poisson} implies that the zero cell of the STIT tessellation has the same distribution as the zero cell of a stationary Poisson hyperplane tessellation with the same intensity measure, which has been studied in, for instance, \cite{Horrmann, Hug}. 
For the Mondrian process, the precise distribution was given in \cite[Proposition 1]{mourtada2020minimax}, though this was generally considered known to the stochastic geometry community. The intersections of the axis-aligned STIT with span$(e_i)$, $i=1, \ldots, d$ are independent one-dimensional Poisson point processes of intensity $\frac{\lambda}{d}$, implying \eqref{e:Z0}. In particular, \eqref{e:Z0} implies that the volume of $Z^{\lambda}_0$ is the product of $d$ i.i.d. positive random variables distributed as the sum of two independent exponential random variables. 
For more general intensity measures, the volume of the zero cell is more difficult to describe. The volume in the isotropic case was studied in \cite{Horrmann}, but only upper and lower bounds on the moments have been obtained. Obtaining this explicit volume would be necessary to generalize both Theorem \ref{thm:density} and the minimax theorem in \cite{mourtada2020minimax} from the Mondrian case to arbitrary STIT.
\end{rem}

\subsection{Consistency of STIT Regression and Density Estimators}

Consider now the random STIT forests described at the beginning of Section \ref{sec:density} in the following regression setting. Let $\{(X_1, Y_1), \ldots, (X_n, Y_n)\}$ be i.i.d. random variables distributed as $(X,Y)$ where $X \in W$ for some bounded window $W \subset \RR^d$ and $Y \in \RR$ is such that $\EE Y^2 < \infty$ and $Y = f(X) + \ee$. The STIT tree regression estimator of $f$ is given by
\begin{align}\label{e:freg}
\tilde{f}_{\lambda, n}(x) = \sum_{i=1}^n \frac{1_{\{X_i \in Z_x^{\lambda}\}}Y_i}{\sum_{i=1}^n 1_{\{X_i \in Z_x^{\lambda}\}}},
\end{align}
and the STIT forest regression estimator is
\begin{align}\label{e:freg}
\tilde{f}_{\lambda, n,M}(x) = \frac{1}{M} \sum_{m=1}^M \tilde{f}^{(m)}_{\lambda, n}(x),
\end{align}
where $\tilde{f}_{\lambda,n}^{(m)}$ are i.i.d. copies of $\tilde{f}_{\lambda, n}$. The estimator is $L^2$ consistent if the $L^2$ loss function converges to zero as the $n$ grows, i.e.
\begin{align*}
\lim_{n \to \infty} \EE[(\tilde{f}_{\lambda,n,M}(X) - f(X))^2] = 0.
\end{align*}
In \cite{mourtada2020minimax}, the authors prove $L^2$ consistency and minimax rates of convergence for this estimator when the STIT is a Mondrian process.
Their proof of consistency using Theorem 4.2 in \cite{Gyorfi} can be extended to the general STIT forest estimator using stochastic geometry as follows.
\begin{theorem}\label{thm:regconsistency}
Let $M \geq 1$ and consider the STIT forest regression estimator $\hat{f}_{\lambda_n, n, M}$ as in \eqref{e:freg} for a sequence of lifetime parameters $\{\lambda_n\}_{n \in \mathbb{N}}$ satisfying $\lambda_n \to \infty$ and $\lambda_n^d/n \to 0$ as $n \to \infty$. This forest is $L^2$ consistent for any $M \geq 1$.
\end{theorem}
\begin{proof}
Following the proof of Theorem 1 in \cite{mourtada2020minimax}, it suffices to prove the following for a sequence of STIT tessellations $Y(\lambda_n)$: 
\begin{itemize}
    \item[(1)] Diam($Z^{\lambda_n}_x$) $ \to 0$ in probability as $n \to \infty$
    \item[(2)] $N_{\lambda_n}/n \to \infty$ in probability as $n \to \infty$,
\end{itemize}
where $N_{\lambda}$ is the number of cells of the STIT $Y(\lambda)$ intersecting $W$. 
First, by the proof of Theorem 2 in \cite{HugSchneider2007}, for any fixed $\tau \in (0,1)$ and for all $\lambda$ large enough, there exists a constant $c := c(\delta, \tau, \Pi)$ such that
\[\PP\left(\text{Diam}(Z^{\lambda}_x) \geq \delta\right) \leq c e^{- \lambda \tau \delta h_{\min}(\Pi)},\]
where $\Pi$ is the associated zonoid of $\Lambda$.  
This generalizes Corollary 1 in \cite{mourtada2020minimax} to the general STIT case. The assumptions on $\lambda_n$ give $(1)$.

Second, by \eqref{e:campbell} we obtain the following upper bound on the expected number of cells $N_{\lambda}$ in the STIT $Y(\lambda)$ that intersect $W$:
\begin{align*}
   \EE[N_{\lambda}] &\leq \EE\left[\sum_{c \in \mathrm{Cells}(Y(\lambda))} 1_{\{c \cap B_R \neq \emptyset\}}\right]  = \EE[V(Z)]^{-1}\EE\left[\int_{\RR^d} 1_{\{Z + y \cap B_R \neq \emptyset\}} \dint y\right] \\
   &= V(\Pi_{\lambda})\EE[V(Z + B_R)] = V(\Pi_{\lambda}) \sum_{k=0}^d \kappa_{d-k}R^k\EE[V_k(Z)].
\end{align*}
where $\Pi_{\lambda}$ is the associated zonoid of the STIT $Y(\lambda)$ and $V_k$ are the intrinsic volumes, see Section 14.3 in \cite{weil}. In particular, $V_d = V$ and $V_k$ is homogeneous of degree $k$. The third equality follows from the fact that $Z + y \cap B_R \neq \emptyset$ if and only if $y \in Z + B_R$ and the last equality appears in \cite[(5.16)]{weil} and follows from the Steiner formula. 
Then, by (10.3) and Theorem 10.3.3 in \cite{weil},
\begin{align*}
    \EE V_j(Z) = \frac{d^{(d)}_j}{\gamma^{(d)}} = \frac{V_{d-j}(\Pi_{\lambda})}{V_d(\Pi_{\lambda})},
\end{align*}
where $\Pi_{\lambda}$ is the associated zonoid to the STIT $Y(\lambda)$, which implies
\begin{align*}
   \EE[N_{\lambda}] & \leq \sum_{k=0}^d \kappa_{d-k} V_{d-k}(\Pi_{\lambda}) = \sum_{k=0}^d \lambda^k \kappa_k  V_k(\Pi_1).
\end{align*}
This generalizes Proposition 2 in \cite{mourtada2020minimax}, and the assumptions on $\lambda_n$ imply $(2)$.
\end{proof}

The same arguments in the proof Theorem \ref{thm:regconsistency} imply the $L^1$ consistency of the STIT forest density estimator, see Theorem 7.2 in \cite{GyorfiDensity}. An estimator $\tilde{f}_n$ is $L^1$ consistent if 
\[\lim_{n \to \infty} \EE[|\tilde{f}_n(X) - f(X)|] = 0.\]
Density estimation benefits from this type of consistency through the relationship between the $L^1$ norm and the total variation distance between probability distributions, see Section 7.1 in \cite{GyorfiDensity}.

\begin{cor}\label{cor:densityconsistency}
Let $M \geq 1$. Consider the STIT forest density $f_{\lambda_n, n, M}$ as in \eqref{e:ftree} for a sequence of lifetime parameters $\{\lambda_n\}_{n \in \mathbb{N}}$ satisfying $\lambda_n \to \infty$ and $\lambda_n^d/n \to 0$ as $n \to \infty$. This forest is $L^1$ consistent for any $M \geq 1$.
\end{cor}

\section{Conclusion and future work}\label{sec:conclusion}

The results in this paper motivate many more questions in stochastic geometry and machine learning. We conclude with some explicit problems. 

First, can we translate these new understandings of STIT and the Mondrian into computational advantages? In the spirit of Theorem \ref{thm:projection}, for given dataset, what would be the optimal cut directions for regression or classification with the corresponding Mondrian forests? 

Another computational question regards the convergence rate for the STIT kernel. In Theorem \ref{thm:rate}, the upper bound depends on an unknown constant coming from the tail bound on the diameter of the typical cell in \eqref{e:diameter}. This bound comes from \cite{HugSchneider2007} where only the asymptotic exponential rate is shown. It is an open problem to obtain an explicit constant for this tail bound for the diameter or a related quantity, 
which would allow computation of the $M$ needed to achieve uniform approximation error $\delta$. 

Second, what is the analogue of Theorem \ref{thm:projection} for infinitely many cut directions? In general, the intersection of a STIT tessellation with intensity measure $\Lambda$ is a STIT tessellation with intensity measure $\Lambda'$, where $\Lambda'$ can be worked out explicitly from $\Lambda$, see \cite[(4.61), p.156]{weil}. If $\Lambda$ has a continuous component, then so does $\Lambda'$, so in particular, it is \emph{not} true that a STIT tessellation with uncountably many cut directions is a projection of the axis-aligned Mondrian. However, one could ask whether there is another ``mother" distribution for the continuous case, that is, can any continuous $\Lambda$ be derived as a linear transformation of a fixed $\Lambda^*$? 

Similarly, Theorem \ref{thm:monotone} shows that many, but not all positive definite kernels that can be approximated with a stationary STIT process. In particular, the Gaussian kernel does not satisfy the condition in Theorem \ref{thm:monotone} since $e^{-t^2}$ is not a completely monotone function. It is an open question to determine whether there are other kinds of additional randomness to incorporate into the intensity measure to obtain random features that can approximate any positive definite stationary kernel.

Third, Theorem \ref{thm:regconsistency} only proves consistency of the STIT regression, but for the Mondrian case, a minimax rate was obtained \cite{mourtada2017universal,mourtada2020minimax}. An extension of the minimax rates for STIT regression estimators is thus a very interesting open problem. As discussed in Remark \ref{rem:more.general}, a key missing ingredient is the distribution of a cell of the tessellation containing a fixed point or its volume. Resolving this open question in stochastic geometry would also lead to a direct generalization of Theorem \ref{thm:density} from Mondrian to STIT density estimators.

Finally, we note that random features for non-stationary positive definite kernel can also be obtained with a STIT process generated from a non-stationary intensity measure $\Lambda$. One example is translation regular intensity measure $\Lambda$. That is, there is a stationary measure $\tilde{\Lambda}$ and a non-negative locally integrable function $\eta$ on $\mathcal{H}^d$ such that $\Lambda(A) = \int_A \eta \tilde{\Lambda}$ for $A \in \mathcal{B}(\mathcal{H}^d)$. Define $g: \mathbb{S}^{d-1} \times (0, \infty) \to (0, \infty)$ by $g(u, t) := \eta(H(u,t))$. We can then define the intensity measure as
\[\Lambda(A) = 2\int_{S^{d-1}}\int_0^{\infty} 1_{A}(H(u,t))g(u,t)dt\phi(du),\]
see Section 11.3 in \cite{weil}. For a STIT tessellation corresponding to this intensity measure, for $x, y \in \RR^d$,
\begin{align*}
    &\PP(x, y \text{ in the same cell}) = \exp(-\Lambda([[x,y]])) \\ &\qquad = \exp\left(- 2 \int_0^{1} \int_{S^{d-1}} |\langle y-x, v \rangle| g(v, \langle v, x + t(y-x)\rangle) d\phi(v) dt\right).
\end{align*}
Thus a concrete problem would be to characterize all non-stationary kernels that a STIT tessellation could approximate.


\bibliographystyle{siamplain}
\bibliography{biblio}
\end{document}